\documentclass[runningheads,orivec]{llncs}

\usepackage[T1]{fontenc}
\usepackage[utf8]{inputenc}
\usepackage{graphicx}
\usepackage{amsmath,amssymb,mathtools}
\usepackage{booktabs}
\usepackage{tabularx}
\usepackage{array}
\usepackage{multirow}
\usepackage{algorithm}
\usepackage{algpseudocode}
\usepackage{xcolor}
\usepackage{xurl}
\usepackage{xspace}
\usepackage{enumitem}
\usepackage{placeins}
\usepackage[hidelinks]{hyperref}
\newcolumntype{Y}{>{\raggedright\arraybackslash}X}
\newcolumntype{L}[1]{>{\raggedright\arraybackslash}p{#1}}
\DeclareMathOperator*{\argmax}{arg\,max}
\DeclareMathOperator{\clip}{clip}

\newcommand{\MERA}{\textsc{PYTHALAB}\discretionary{-}{}{-}\textsc{MERA}\xspace}
\newcommand{\GRACE}{\textsc{GRACE}\xspace}
\newcommand{\REFINE}{\ensuremath{\mathrm{REFINE}_{\mathrm{B1}}}\xspace}

\begin{document}

\title{PYTHALAB-MERA: Validation-Grounded Memory, Retrieval, and Acceptance Control for Frozen-LLM Coding Agents}
\titlerunning{PYTHALAB-MERA: Memory and Acceptance Control}

\author{Mehmet Iscan}
\authorrunning{M. Iscan}
\institute{PythaLab, Yildiz Technical University\\\email{miscan@yildiz.edu.tr}}

\maketitle

\begin{abstract}
Local LLM-based coding agents increasingly work in settings where correctness is earned through execution feedback, persistent state, and bounded repair, not through a single fluent answer. Static retrieval, long-context prompting, self-refinement, execution-feedback repair, and reinforcement learning over model weights each address part of this setting, but they do not jointly provide validation-grounded episodic memory, adaptive retrieval-action selection, delayed credit assignment, and structural skill reuse around a frozen local model. We introduce \MERA, a lightweight external controller for local validation-conditioned code generation. The frozen language model proposes complete source files; the controller decides which memory records and AST-derived skills should enter the next prompt, validates each candidate through a fail-fast pipeline, converts validation outcomes into bounded shaped rewards, and propagates delayed credit through TD(\(\lambda\))-style eligibility traces. We evaluate the implementation as a local CLI artifact on reinforcement-learning coding tasks with strict validation gates. In the measured hard RL setting with three tasks, three repetitions, and a three-attempt budget, \MERA passed 8/9 strict validations; the self-refinement baseline and the investigated \GRACE{} extension each passed 0/9. These results support a deliberately bounded claim: in this recorded setting, the external memory-and-retrieval controller improved validation success. They do not establish general-purpose code synthesis, state-of-the-art performance, formal program correctness, or formal safety.
\end{abstract}

\section{Introduction}

LLM-based coding agents are increasingly used in long-horizon, tool-mediated software workflows in which a plausible single output is not sufficient. A generated program must satisfy interface contracts, import constraints, runtime behavior, and task-specific tests; failures at one step influence the context required at later steps. Recent work on stateful code-agent benchmarks, software-engineering agents, and terminal-native coding systems shows that agent performance depends strongly on maintaining task state, execution evidence, and structured context rather than merely extending the prompt with more text \cite{i1_1,d3_1,i6_2,i6_6}. In this setting, the central problem is not only code generation but also control: the agent must decide which evidence to preserve, which error signals to trust, which prior attempts to recall, and when a candidate has reached an externally defined acceptance condition.

The limitations of single-shot generation are now well documented. Iterative self-repair can improve pass rates over one-pass generation, but unguided self-correction remains vulnerable to model-internal hallucination and weak feedback grounding \cite{i2_1,e1_3}. Execution-feedback methods provide a stronger signal by returning compiler diagnostics, runtime failures, unit-test outcomes, or behavioral feedback to the model \cite{i3_1,d2_3}. However, many execution-feedback systems still use static prompt rules or manually designed repair loops. Conversely, reinforcement-learning approaches for code generation often update the generator itself using execution-derived rewards \cite{e1_4,m4_3}, which can be computationally expensive, unstable for small local traces, and difficult to deploy in privacy-preserving local settings. This creates a gap for architectures that use validation feedback adaptively while keeping the language model frozen.

Memory and retrieval add a second dimension to the problem. Retrieval-augmented program repair, AST-guided repair, repository-level RAG, and experiential memory systems show that prior code, failure traces, and structural context can improve repair or maintenance workflows \cite{i5_1,i5_4,i5_6,d1_1,d1_3}. Yet passive nearest-neighbor retrieval is not enough when usefulness depends on the current failure mode, attempt index, validation progress, and downstream outcome. Agentic RAG has therefore been framed as a sequential decision process in which retrieval is a state-dependent action rather than a preprocessing step \cite{d1_2}. \MERA follows this interpretation, but specializes it to local code repair: retrieval actions select failure-matched examples, AST-structural matches, skills, or no context, and these actions are updated using validation-derived rewards.

The same persistence that enables memory reuse also introduces risks. Memory poisoning, stale records, prompt injection through recalled evidence, and leakage from private retrieval stores are documented concerns for long-lived LLM agents \cite{i4_1,i4_2,i4_4,d5_3}. Therefore, a memory-augmented coding agent should not treat recalled snippets as authoritative. In \MERA, recalled episodes and extracted skills are inserted only as untrusted prompt evidence; correctness is externalized through a fail-fast validation pipeline. This does not provide formal safety, but it prevents memory similarity or model confidence from serving as the acceptance criterion.

This paper presents a local validation-conditioned coding-agent architecture, \MERA, centered on a lightweight external controller for a frozen LLM. The method is built around four principles. First, generated code is accepted only if it reaches a validator-defined zero-cost terminal set. Second, memory retrieval is a contextual action selected by a LinUCB controller, not a fixed nearest-neighbor lookup. Third, validation outcomes are converted into bounded shaped rewards and propagated backward through the repair trajectory using TD(\(\lambda\))-style delayed credit. Fourth, successful programs are mined into an AST-derived skill library, but skills remain untrusted until a future generated candidate passes validation.

The contributions are as follows. (i) We formalize local code repair with a frozen LLM as a finite-horizon validation-conditioned control problem over prompt evidence, memory retrieval, reward, and acceptance. (ii) We implement a controller that combines deterministic episodic fingerprints, LinUCB retrieval-action selection, validation-shaped rewards, eligibility-weighted delayed credit, and AST-derived skill reuse. (iii) We state a narrow zero-cost acceptance property that ties acceptance to the implemented validator-defined terminal set. (iv) We report local benchmark evidence against a self-refinement baseline and an investigated \GRACE{} extension. (v) We keep the empirical claim deliberately bounded, separating the supported local result from unsupported claims about general-purpose synthesis, formal correctness, formal safety, or state-of-the-art performance.

We use the project-qualified name \MERA rather than the generic MERA shorthand to make the proposed artifact easier to distinguish from unrelated uses of the acronym. The name emphasizes the actual contribution: a PythaLab implementation of memory-enhanced retrieval and validator-defined acceptance control around a frozen local generator.

We use \GRACE to denote an investigated additive extension, the Gap-aware Repair And Consolidation Engine. In this implementation, \GRACE adds intent--execution-gap repair guidance, AST-diff repair suggestions, and score-thresholded transition consolidation on top of the \MERA controller. It is not the proposed final method; it is retained as a negative-result comparison arm.

\section{Method}

\subsection{Methodological framing}

\MERA is a reinforcement-learning-inspired side controller for local validation-conditioned code generation. The frozen language model is denoted by \(G_\omega\), where \(\omega\) are fixed parameters. Unlike model-tuning approaches for execution-grounded code generation \cite{m4_3,e1_4,d2_1}, \MERA does not update the generator:
\begin{equation}
\omega_{t+1}=\omega_t=\omega .
\label{eq:frozen}
\end{equation}
The adaptive object is an external controller state,
\begin{equation}
\Theta_t = \left(B_t^{\mathrm{ret}},M_t,L_t,\mathcal{T}_t,B_t^{\mathrm{dec}}\right),
\label{eq:theta}
\end{equation}
where \(B_t^{\mathrm{ret}}\) is the LinUCB retrieval-action controller, \(M_t\) is episodic memory, \(L_t\) is the AST-derived skill library, \(\mathcal{T}_t\) is the delayed-credit trace state, and \(B_t^{\mathrm{dec}}\) is an optional decoding-profile bandit. The main learning problem is therefore not token-level language modeling but controller-level adaptation over prompt conditioning, memory, and validation feedback. This separation is aligned with agent architectures that externalize memory and control rather than embedding all adaptation into the LLM \cite{m1_1,m6_1,m11_1}.

\begin{figure}[!htbp]
\centering
\includegraphics[width=\textwidth]{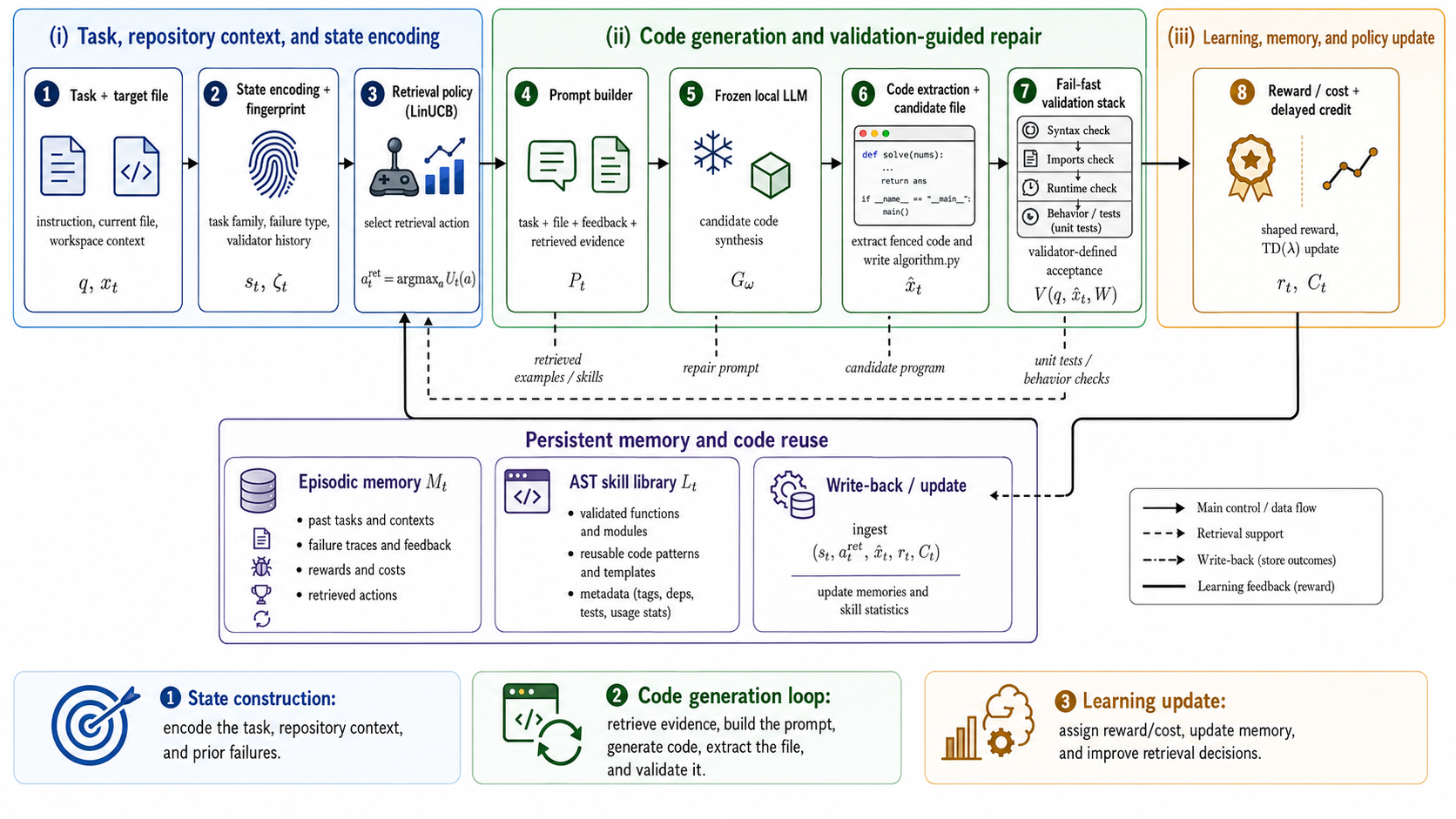}
\caption{\MERA validation-grounded control loop. The frozen generator proposes a candidate program, while the external controller selects retrieval evidence, validates the candidate, and converts validation feedback into memory and credit updates.}
\label{fig:architecture}
\end{figure}

Figure~\ref{fig:architecture} separates generation from control. The task description and current target file are first compressed into a state fingerprint. A LinUCB retrieval controller then chooses the prompt-evidence action: no retrieval, failure-matched episodes, AST-structural matches, reusable skills, or a small mixture of these sources. Retrieved episodes and skills are treated as untrusted context, not as correctness evidence. The frozen LLM receives the composed prompt and emits a candidate implementation. Acceptance is decided only by the fail-fast validation pipeline. The resulting success, partial progress, failure type, and runtime cost are converted into bounded reward/cost signals; these signals update the retrieval controller, episodic memory, AST-skill store, and TD($\lambda$)-style delayed-credit trace. Thus, \MERA learns how to condition future prompts without changing the LLM weights.

\subsection{Task, state, and action}

A task instance consists of a natural-language request \(q\), a local workspace \(W\), and a target implementation file \(x_t\). In the implemented artifact, the target file is a single Python file named \texttt{algorithm.py}. The agent has an attempt budget \(T\). At attempt \(t\), the controller observes
\begin{equation}
 s_t = \left(q,x_t,H_t,F_{t-1},M_t,L_t,c_t\right),
\label{eq:state}
\end{equation}
where \(H_t\) is retained interaction history, \(F_{t-1}\) is the previous validation report, \(M_t\) is episodic memory, \(L_t\) is the skill library, and \(c_t\) is a compact context key encoding task family, model identity, edit mode, previous failure type, previous validation progress, and a coarse duration bucket.

The action is decomposed into optional decoding control, retrieval control, and generated code:
\begin{equation}
 a_t=\left(a_t^{\mathrm{dec}},a_t^{\mathrm{ret}},a_t^{\mathrm{gen}}\right).
\label{eq:action}
\end{equation}
The core \MERA decision is \(a_t^{\mathrm{ret}}\), which determines whether to condition the next prompt on failure-matched memories, AST-structural matches, skills, a mixture of these, or no retrieved context. The prompt composer is deterministic:
\begin{equation}
 P_t = \Gamma(q,x_t,H_t,F_{t-1},R_t,S_t),
\label{eq:prompt}
\end{equation}
where \(R_t\) denotes selected episodic records and \(S_t\) selected skills. The frozen model then produces text and the controller extracts a Python candidate:
\begin{equation}
 y_t = G_\omega(P_t), \qquad \hat{x}_t=\mathcal{E}(y_t),
\label{eq:generate_extract}
\end{equation}
where \(\mathcal{E}\) is a fenced-code extraction operator. If extraction succeeds, the evaluated direct loop materializes \(\hat{x}_t\) as the target file before validation:
\begin{equation}
 x_{t+1}=\hat{x}_t.
\label{eq:materialize}
\end{equation}
Thus, \MERA is validation-gated for acceptance, but the evaluated loop is not a pre-materialization staging system. This distinction matters because acceptance safety and write-time safety are different claims.

As a design objective for organizing the controller signals, the finite-horizon controller can be written as a discounted cost minimization problem,
\begin{equation}
 \min_{\pi_\Theta}\;J(\pi_\Theta)=\mathbb{E}_{\pi_\Theta,G_\omega}\left[\sum_{t=0}^{T-1}\gamma^t C_t\right], \qquad 0\leq\gamma\leq 1,
\label{eq:objective_cost}
\end{equation}
where \(C_t\) denotes a validation-defined cost. The implementation does not require an exact dynamic-programming solution of this objective. Instead, the online learners receive a bounded shaped reward signal:
\begin{equation}
 \max_{\pi_\Theta}\;\mathbb{E}_{\pi_\Theta,G_\omega}\left[\sum_{t=0}^{T-1}\gamma^t r_t\right].
\label{eq:objective_reward}
\end{equation}
The reward view in Eq.~\eqref{eq:objective_reward} is equivalent to the cost view only when \(r_t\) is an affine decreasing transform of \(C_t\). In the implemented system, \(r_t\) is a pragmatic validation-shaped learning signal, defined below, while \(C_t^{\mathrm{acc}}\) is used for the acceptance property. The implemented learner therefore applies online bandit and delayed-credit updates rather than solving the finite-horizon objective globally.

The transition dynamics are induced by the black-box generator, workspace, and validator. \MERA does not learn a transition model over programs and does not claim optimal control over the full program-generation process.

\subsection{Preprocessing and data preparation}

The data used by the controller are produced online from each repair attempt. There is no offline supervised training set and no model-weight update. Instead, each attempt is converted into a compact control record through four preprocessing steps. First, the model response is parsed by the extractor $\mathcal{E}$; non-extractable responses are converted into an extraction-failure report rather than silently discarded. Second, the target file or candidate code is parsed into AST-derived structural features. Third, the validator report is normalized into a primary failure class, passed-check count, duration, and behavior-failure flag. Fourth, the task and code are converted into the fingerprint $\zeta_t$ used for retrieval and memory persistence.

This preprocessing stage is important because it separates raw text from control evidence. The controller never learns from an unconstrained natural-language transcript alone. It learns from normalized objects: task-family labels, token 3-grams, AST features, failure signatures, shaped rewards, and selected actions. These records form the data stream for episodic memory, LinUCB updates, skill extraction, and delayed-credit propagation.

\subsection{Validation and zero-cost acceptance}

The validator maps a generated candidate to a structured report:
\begin{equation}
 F_t=V(q,\hat{x}_t,W)=\left(z_{t,1},z_{t,2},\ldots,z_{t,m_t}\right).
\label{eq:validation_report}
\end{equation}
Here \(m_t\) ranges over checks that were executed or explicitly skipped before the fail-fast pipeline terminated.
The implemented validation order is fail-fast:
\begin{equation}
\text{syntax}\rightarrow\text{undefined-name}\rightarrow\text{spec-contract}\rightarrow\text{import}\rightarrow\text{runtime}\rightarrow\text{behavior}.
\label{eq:validation_order}
\end{equation}
Behavior checks are used when the task exposes executable behavioral obligations. Define
\begin{equation}
 u_{t,k}=\begin{cases}
1, & z_{t,k}\text{ passed or was skipped},\\
0, & z_{t,k}\text{ failed}.
\end{cases}
\label{eq:stage_indicator}
\end{equation}
The primary validator-pass indicator is
\begin{equation}
 V_t=\mathbf{1}\left[\prod_{k=1}^{m_t}u_{t,k}=1\ \land\ \operatorname{primary\_failure}(F_t)=\mathrm{UNKNOWN}\right].
\label{eq:validator_pass}
\end{equation}
When an optional post-validation semantic judge is enabled, it can veto a validator-passing candidate but cannot create success. Let \(J_t\in\{0,1\}\) indicate absence of a high-confidence judge veto: \(J_t=0\) only for a high-confidence judge failure, and \(J_t=1\) when the judge is disabled, skipped, passing, uncertain, or low-confidence. The final acceptance indicator is
\begin{equation}
 A_t=V_tJ_t.
\label{eq:acceptance}
\end{equation}
The terminal acceptance cost is
\begin{equation}
 C_t^{\mathrm{acc}}=1-A_t,
\label{eq:acc_cost}
\end{equation}
and the zero-cost terminal set is
\begin{equation}
 \mathcal{X}_0(q,W)=\{\hat{x}: C^{\mathrm{acc}}(\hat{x};q,W)=0\}.
\label{eq:zero_set}
\end{equation}

\begin{proposition}[Zero-cost acceptance equivalence]
For every attempt \(t\), \(C_t^{\mathrm{acc}}=0\) if and only if \(A_t=1\).
\end{proposition}
\begin{proof}
By Eq.~\eqref{eq:acc_cost}, \(C_t^{\mathrm{acc}}=1-A_t\). Since \(A_t\in\{0,1\}\), the cost is zero exactly when \(A_t=1\). Eq.~\eqref{eq:acceptance} expands this event into validator acceptance and absence of a judge veto. \(\square\)
\end{proof}

This proposition is a definitional property, not a convergence theorem. It states what it means for the evaluated system to reach zero cost; it does not imply that every task will reach that set.

A progress-sensitive diagnostic stage cost can also be defined over the observed fail-fast report:
\begin{equation}
 C_t^{\mathrm{stage}}=1-\frac{1}{m_t}\sum_{k=1}^{m_t}u_{t,k}.
\label{eq:stage_cost}
\end{equation}
This distinguishes shallow failures such as extraction or syntax errors from deeper failures that pass early validation stages and fail only behavioral checks. Treating skipped checks as passed is used only for this diagnostic progress cost; acceptance remains governed by the validator report and optional judge gate. The implementation uses the non-skipped passed-check count \(p_t\), rather than \(C_t^{\mathrm{stage}}\) directly, as the progress term in the shaped reward below.

\subsection{Validation-shaped reward}

The shaped reward follows the implementation-level design used for online attempt updates and episode summaries. Let \(p_t\) be the number of passed non-skipped checks, \(e_t\) an extraction-failure indicator, \(b_t\) a behavior-failure indicator, \(d_t\) validation duration, and \(D\) a latency normalization horizon. The bounded reward is
\begin{equation}
 r_t=\clip_{[-1,1]}\left(R_{\mathrm{succ}}A_t+(1-A_t)\beta p_t-\eta t-\kappa e_t-\mu b_t-\xi\min\left(1,\frac{d_t}{D}\right)\right).
\label{eq:reward}
\end{equation}
The terms reward terminal success, give partial credit for validation progress, penalize later attempts, penalize extraction and behavioral failures, and optionally penalize latency. Prior code-generation RL methods similarly convert compiler, runtime, test, or structural signals into shaped rewards \cite{m4_3,m4_5,e1_4}; \MERA uses the signal for an external retrieval controller rather than for generator fine-tuning. For Beta/Thompson-style updates that require a value in \([0,1]\), the reward is mapped to a pseudo-success:
\begin{equation}
 \tilde{r}_t=\frac{r_t+1}{2},\qquad \tilde{r}_t\in[0,1].
\label{eq:pseudo_success}
\end{equation}
LinUCB consumes the clipped reward \(r_t\) directly for immediate post-attempt updates. Delayed LinUCB updates consume the clipped TD-style delta \(\delta_j\) with a nonnegative update weight.

\subsection{Episodic memory and fingerprints}

Each attempt is stored as a structured episodic record. The fingerprint is
\begin{equation}
 \zeta_t=\left(f(q),g_3(q),\psi_{\mathrm{AST}}(\hat{x}_t),\sigma(F_t),h(q,\hat{x}_t)\right),
\label{eq:fingerprint}
\end{equation}
where \(f(q)\) is a coarse task-family label, \(g_3(q)\) is a bounded token 3-gram representation, \(\psi_{\mathrm{AST}}\) is an AST feature vector, \(\sigma(F_t)\) is a normalized failure signature, and \(h\) is a complexity bucket. The AST vector contains function count, class count, loop depth, recursion indicators, class usage, common-library indicators, state-machine indicators, and an approximate cyclomatic-complexity feature. AST representations and structural fingerprints are well established in code similarity, clone detection, and fault localization \cite{m7_2,m7_3,m7_4}.

An episodic record is
\begin{equation}
 e_i=\left(\zeta_i,q_i,\hat{x}_i,F_i,r_i,\alpha_i,d_i,a_i^{\mathrm{dec}},a_i^{\mathrm{ret}}\right).
\label{eq:episode_record}
\end{equation}
Here, \(\alpha_i\in\{0,1\}\) records whether attempt \(i\) was accepted. Both successful and failed attempts are retained. This is important because a failed attempt marks part of the boundary around the zero-cost set: an invalid import, a missing interface, a type error, or a runtime crash may guide later retrieval decisions. Episodic memory has been used for cross-task adaptation in LLM agents and debugging systems \cite{m6_1,m6_5}; \MERA specializes the record schema for validation-conditioned coding.

Similarity between a query fingerprint \(\zeta_t\) and stored record \(\zeta_i\) is
\begin{equation}
 S(\zeta_t,\zeta_i)=\frac{w_{\mathrm{tok}}S_{\mathrm{tok}}+w_{\mathrm{ast}}S_{\mathrm{ast}}+w_{\mathrm{fail}}S_{\mathrm{fail}}+w_{\mathrm{fam}}S_{\mathrm{fam}}}{w_{\mathrm{tok}}+w_{\mathrm{ast}}+w_{\mathrm{fail}}+w_{\mathrm{fam}}}.
\label{eq:similarity}
\end{equation}
The configured default weights are \((0.4,0.3,0.2,0.1)\). The AST term is decomposed as
\begin{equation}
 S_{\mathrm{ast}}=0.70S_{\mathrm{struct}}+0.15S_{\mathrm{imports}}+0.15S_{\mathrm{return}}.
\label{eq:ast_similarity}
\end{equation}
This retrieval is deterministic and embedding-free, which supports local reproducibility and avoids dependence on an external vector database.

\subsection{Retrieval-action controller}

\MERA learns which type of evidence should enter the prompt. The retrieval-action set is
\begin{equation}
\begin{split}
\mathcal{A}_{\mathrm{ret}}=\{&\texttt{none},\texttt{1\_failure\_match},\texttt{1\_ast\_match},\texttt{1\_failure\_1\_ast},\\
&\texttt{2\_ast\_match},\texttt{1\_skill\_only},\texttt{1\_failure\_1\_skill},\texttt{diff\_only}\}.
\end{split}
\label{eq:action_set}
\end{equation}
This design follows adaptive information retrieval and contextual-bandit retrieval work in which retrieval utility is conditioned on state and feedback rather than only nearest-neighbor similarity \cite{m3_1,m3_3,m3_4}. At attempt \(t\), the controller computes
\begin{equation}
 \phi_t=\Phi(s_t)\in\mathbb{R}^{d},
\label{eq:features}
\end{equation}
with \(d=16\) in the implementation. The vector includes a bias term, normalized attempt index, previous validation progress, failure-type indicators, and duration buckets.

The update rule follows the standard ridge-regularized LinUCB form for contextual bandits \cite{f1_1}. For each retrieval action \(a\), LinUCB maintains
\begin{equation}
 A_a\in\mathbb{R}^{d\times d}, \qquad b_a\in\mathbb{R}^{d},
\label{eq:linucb_state}
\end{equation}
with ridge initialization
\begin{equation}
 A_a^{(0)}=\lambda I_d,\qquad b_a^{(0)}=0,\qquad \lambda>0.
\label{eq:linucb_init}
\end{equation}
The estimated utility parameter is
\begin{equation}
 \hat{\theta}_a=A_a^{-1}b_a,
\label{eq:theta_hat}
\end{equation}
and the upper-confidence score is
\begin{equation}
 U_t(a)=\hat{\theta}_a^\top\phi_t+\alpha_{\mathrm{ucb}}\sqrt{\phi_t^\top A_a^{-1}\phi_t}.
\label{eq:ucb}
\end{equation}
The controller selects
\begin{equation}
 a_t^{\mathrm{ret}}=\argmax_{a\in\mathcal{A}_{\mathrm{ret}}}U_t(a).
\label{eq:select}
\end{equation}
After validation, if the selected context was actually injected or the selected action was \texttt{none}, the update is
\begin{equation}
 A_{a_t}\leftarrow A_{a_t}+w_t\phi_t\phi_t^\top,
\label{eq:a_update}
\end{equation}
\begin{equation}
 b_{a_t}\leftarrow b_{a_t}+w_t r_t\phi_t.
\label{eq:b_update}
\end{equation}
Direct post-attempt updates use \(w_t=1\); delayed-credit updates use fractional weights. Lightweight contextual bandits are appropriate here because the action set is small, evaluation budgets are bounded, and local traces are too small to justify full online LLM reinforcement learning \cite{m2_1,m2_4,m2_5,m11_1}.

\begin{proposition}[Positive-definite retrieval state]
If \(A_a^{(0)}=\lambda I_d\) with \(\lambda>0\) and all update weights satisfy \(w_t\geq 0\), then \(A_a\) remains positive definite after every update.
\end{proposition}
\begin{proof}
For any nonzero vector \(z\),
\begin{equation}
z^\top(A_a+w_t\phi_t\phi_t^\top)z=z^\top A_a z+w_t(z^\top\phi_t)^2.
\label{eq:pd_proof}
\end{equation}
The first term is positive by induction and the second is nonnegative. Hence the updated matrix remains positive definite and invertible. \(\square\)
\end{proof}

\subsection{Delayed credit propagation}

A retrieval action can influence a later repair attempt even if the immediate attempt fails. \MERA stores the trajectory
\begin{equation}
 \tau=\left[(s_0,a_0,r_0),(s_1,a_1,r_1),\ldots,(s_{T'},a_{T'},r_{T'})\right],\qquad T'\leq T-1.
\label{eq:trajectory}
\end{equation}
For target step \(j\), the TD-style error is
\begin{equation}
 \delta_j=\clip_{[-\Delta,\Delta]}\left(r_j+\gamma V_\nu(s_{j+1})-V_\nu(s_j)\right).
\label{eq:td_error}
\end{equation}
The runtime integration is conservative: no separate learned value estimator is maintained, so \(V_\nu(s_j)=V_\nu(s_{j+1})=0\), reducing the update to clipped reward propagation:
\begin{equation}
 \delta_j=\clip_{[-\Delta,\Delta]}(r_j).
\label{eq:clipped_reward}
\end{equation}
Eligibility evolves as
\begin{equation}
 E_i^{(j)}=\gamma\lambda_{\mathrm{td}}E_i^{(j-1)}+\mathbf{1}[i=j],
\label{eq:eligibility}
\end{equation}
and the update weight is
\begin{equation}
 w_{i,j}=\min(w_{\max},\alpha_{\mathrm{td}}E_i^{(j)}).
\label{eq:td_weight}
\end{equation}
Here \(0\leq\gamma\leq1\), \(0\leq\lambda_{\mathrm{td}}\leq1\), \(\alpha_{\mathrm{td}}\geq0\), \(\Delta\geq0\), and \(w_{\max}\geq0\). The current implementation uses \(w_{\max}=0.5\) as the fixed helper default for delayed-credit dispatch.
The delayed signal delivered to the side learner is
\begin{equation}
 \bar{r}_{i,j}=w_{i,j}\delta_j.
\label{eq:delayed_reward}
\end{equation}
This mechanism is inspired by eligibility traces in temporal-difference learning \cite{f1_2} and delayed reward handling in RL \cite{m5_4,m5_6}, but the runtime instantiation is deliberately weaker: it performs clipped, eligibility-weighted reward propagation without learning a separate value function and is not presented as a convergence proof for program repair.

\begin{proposition}[Bounded delayed-credit magnitude]
For every delayed-credit update, \(|\bar{r}_{i,j}|\leq w_{\max}\Delta\).
\end{proposition}
\begin{proof}
Eq.~\eqref{eq:td_error} implies \(|\delta_j|\leq \Delta\), and Eq.~\eqref{eq:td_weight} implies \(w_{i,j}\leq w_{\max}\). Therefore \(|\bar{r}_{i,j}|=|w_{i,j}\delta_j|\leq w_{\max}\Delta\). \(\square\)
\end{proof}

\subsection{AST-derived skill library}

When a candidate is accepted, \MERA parses the accepted program and extracts top-level functions and class methods:
\begin{equation}
 \mathcal{U}(\hat{x})=\{u_1,\ldots,u_m\}.
\label{eq:skill_set}
\end{equation}
Each skill is canonicalized by removing docstrings, normalizing local identifiers, recording parameters, and hashing the normalized AST:
\begin{equation}
 h_i=\operatorname{BLAKE2b}\left(\operatorname{dump}(\operatorname{NormalizeAST}(u_i))\right).
\label{eq:skill_hash}
\end{equation}
A skill record is
\begin{equation}
 \ell_i=\left(h_i,B_i,P_i,\mathcal{F}_i,n_i^{\mathrm{offered}},n_i^{\mathrm{succ}},q_i,\tau_i\right),
\label{eq:skill_record}
\end{equation}
where \(B_i\) is canonical body, \(P_i\) parameter metadata, \(\mathcal{F}_i\) task-family set, \(n_i^{\mathrm{offered}}\) offer count, \(n_i^{\mathrm{succ}}\) successful-use count, \(q_i\) quarantine flag, and \(\tau_i\) last-used timestamp. Program-abstraction and skill-library work motivates reusing validated subroutines to amortize search \cite{m8_1,m8_4,m8_6}. \MERA uses skills only as prompt evidence; it does not automatically splice them into the final program.

\subsection{Investigated GRACE extension}

\GRACE is the investigated extension rather than the proposed method. It is useful to state its mechanism because the empirical comparison includes it. Formally, \GRACE keeps the \MERA controller state and adds an untrusted repair-guidance store,
\begin{equation}
 \Theta_t^{\mathrm{GRACE}}=\left(\Theta_t,\mathcal{O}_t,\mathcal{G}_t\right),
\label{eq:grace_state}
\end{equation}
where \(\mathcal{O}_t\) is a persistent set of AST-diff repair operators and \(\mathcal{G}_t\) is a finite set of ephemeral intent--execution-gap hints derived from the previous attempt's reasoning trace and emitted code. Either set may be empty. Neither component has authority to accept code; both are prompt-side hints that remain subordinate to the validator.

The consolidation gate decides whether a transition between two consecutive attempts is informative enough to store. Let \(v_t\) be the validator total score, \(p_t\) the number of passed non-skipped checks, and \(A_t\) the acceptance indicator. With thresholds \(\Delta_v\geq0\) and \(\Delta_p\geq0\), the gate is
\begin{equation}
\begin{aligned}
 K_t=\mathbf{1}\bigl[&
 A_t=1\ \lor\
 (p_t-p_{t-1}\geq \Delta_p \land v_t\geq v_{t-1})\\
 &{}\lor\
 (v_t-v_{t-1}\geq \Delta_v \land p_t\geq p_{t-1})
 \bigr].
\end{aligned}
\label{eq:grace_gate}
\end{equation}
If \(K_t=1\), \GRACE derives a structural repair operator from the before--after AST difference:
\begin{equation}
 o_t=\mathcal{D}_{\mathrm{AST}}\left(\hat{x}_{t-1},\hat{x}_t,\operatorname{primary\_failure}(F_{t-1}),\operatorname{primary\_failure}(F_t)\right).
\label{eq:grace_operator}
\end{equation}
Operators are later eligible for prompt injection only when their source failure matches the current failure and either their causal offer history is strong enough or a cold-start structural bootstrap condition is met. In simplified form, with \(N_{\mathrm{off}}(o)\) offers and \(N_{\mathrm{succ}}^{\mathrm{off}}(o)\) successful offered outcomes,
\begin{equation}
 \operatorname{eligible}(o)=
 \begin{cases}
  \mathbf{1}\!\left[N_{\mathrm{succ}}^{\mathrm{off}}(o)/N_{\mathrm{off}}(o)\geq \rho\right], & N_{\mathrm{off}}(o)>0,\\
  \mathbf{1}\!\left[\operatorname{bootstrap}(o)=1\right], & N_{\mathrm{off}}(o)=0,
 \end{cases}
\label{eq:grace_eligibility}
\end{equation}
where \(\rho\) is the minimum causal offer-success threshold. The next prompt can receive
\begin{equation}
 G_t^{\mathrm{GRACE}}=\operatorname{TopK}\{o\in\mathcal{O}_t:\operatorname{eligible}(o),\ f_{\mathrm{from}}(o)=\operatorname{primary\_failure}(F_{t-1})\}\cup \mathcal{G}_t .
\label{eq:grace_guidance}
\end{equation}
This formalization explains why \GRACE can be evaluated as an additive extension: it changes prompt evidence and memory consolidation, not the generator, validator, or acceptance rule. The negative hard RL result therefore does not refute \MERA; it shows that adding structural repair-operator guidance and gap hints did not improve the measured setting under the same attempt and validation budget.

\subsection{Safety and execution boundaries}

The model emits text; the controller extracts code, writes the target file, runs validators, records memory, and decides acceptance. Validation commands are executed through explicit argument lists, an allowlist, timeouts, and truncated outputs. Let \(\mathcal{C}_{\mathrm{cmd}}\) be the allowlisted command family. Validation execution satisfies
\begin{equation}
 c_t\in\mathcal{C}_{\mathrm{cmd}},
\label{eq:cmd_allow}
\end{equation}
with runtime and output bounds
\begin{equation}
 d_t\leq d_{\max},\qquad |\mathrm{stdout}_t|+|\mathrm{stderr}_t|\leq B_{\max}.
\label{eq:exec_bounds}
\end{equation}
Untrusted generated code remains a known risk, and sandboxing, isolation, and benchmark validity are active concerns in LLM code evaluation \cite{m9_3,m9_4,e7_2}. Accordingly, the system supports an operational claim only: the workflow is command-bounded and validation-gated. It does not provide formal safety, formal verification, or complete protection against memory poisoning.

\subsection{Algorithm}

Algorithm~\ref{alg:mera} gives the online refinement loop. Algorithm~\ref{alg:td} gives the delayed-credit dispatch.

\begin{algorithm}[t]
\caption{\MERA online validation-conditioned refinement}
\label{alg:mera}
\begin{algorithmic}[1]
\Require Task \(q\), workspace \(W\), initial file \(x_0\), frozen model \(G_\omega\), validator \(V\), memory \(M\), skill library \(L\), LinUCB controller \(B^{\mathrm{ret}}\), optional decoding bandit \(B^{\mathrm{dec}}\), attempt budget \(T\)
\Ensure Accepted implementation \(x^*\) or failed result with validation evidence
\State \(H_0\gets\operatorname{InitializeHistory}(q,x_0)\); \(F_{-1}\gets\emptyset\); \(\tau\gets []\)
\For{\(t=0,1,\ldots,T-1\)}
    \State \(s_t\gets(q,x_t,H_t,F_{t-1},M,L,c_t)\); \(\phi_t\gets\Phi(s_t)\)
    \State Select optional decoding action \(a_t^{\mathrm{dec}}\)
    \State \(a_t^{\mathrm{ret}}\gets\argmax_{a\in\mathcal{A}_{\mathrm{ret}}}U_t(a)\)
    \State \(\zeta_t\gets\operatorname{Fingerprint}(q,x_t,F_{t-1})\)
    \State \((R_t,S_t)\gets\operatorname{RetrieveAndMatch}(M,L,\zeta_t,a_t^{\mathrm{ret}})\)
    \State \(P_t\gets\Gamma(q,x_t,H_t,F_{t-1},\mathcal{S}(R_t,S_t))\)
    \State \(y_t\gets G_\omega(P_t)\); \(\hat{x}_t\gets\mathcal{E}(y_t)\)
    \If{\(\hat{x}_t=\emptyset\)}
        \State \(F_t\gets\operatorname{ExtractionFailureReport}()\); \(r_t\gets\operatorname{Reward}(F_t,t)\)
        \State Append \((s_t,a_t,r_t)\) to \(\tau\); append extraction feedback to \(H_{t+1}\)
        \State \textbf{continue}
    \EndIf
    \State Materialize \(\hat{x}_t\) as \texttt{algorithm.py}
    \State \(F_t\gets V(q,\hat{x}_t,W)\); \(r_t\gets\operatorname{Reward}(F_t,t)\)
    \If{retrieval action is attributable}
        \State Update \(B^{\mathrm{ret}}\) using Eqs.~\eqref{eq:a_update}--\eqref{eq:b_update}
    \EndIf
    \State Persist episodic record \(e_t\) in \(M\); append \((s_t,a_t,r_t)\) to \(\tau\)
    \If{\(A_t=1\)}
        \State Harvest AST-derived skills from \(\hat{x}_t\) into \(L\)
        \State Dispatch delayed credit using Algorithm~\ref{alg:td}
        \State \Return \(\hat{x}_t\)
    \EndIf
    \State Append validator feedback to \(H_{t+1}\); truncate history under context budget
\EndFor
\State Dispatch delayed credit using Algorithm~\ref{alg:td}
\State \Return failed result with final validation report
\end{algorithmic}
\end{algorithm}

\begin{algorithm}[t]
\caption{TD(\(\lambda\))-style delayed credit dispatch}
\label{alg:td}
\begin{algorithmic}[1]
\Require Trajectory \(\tau=[(s_0,a_0,r_0),\ldots,(s_{T'},a_{T'},r_{T'})]\), \(\gamma\), \(\lambda_{\mathrm{td}}\), \(\alpha_{\mathrm{td}}\), clipping bound \(\Delta\), maximum weight \(w_{\max}\)
\Ensure Weighted updates to attributable side learners
\State Initialize \(E_i\gets0\) for each source step
\For{\(j=0,1,\ldots,T'\)}
    \State \(\delta_j\gets\clip_{[-\Delta,\Delta]}(r_j+\gamma V_\nu(s_{j+1})-V_\nu(s_j))\)
    \For{\(i=0,1,\ldots,j\)} \State \(E_i\gets\gamma\lambda_{\mathrm{td}}E_i\) \EndFor
    \State \(E_j\gets E_j+1\)
    \For{\(i=0,1,\ldots,j\)}
        \If{\(E_i\) is below the eligibility floor} \State \textbf{continue} \EndIf
        \State \(w_{i,j}\gets\min(w_{\max},\alpha_{\mathrm{td}}E_i)\)
        \State \(\bar r_{i,j}\gets w_{i,j}\delta_j\)
        \If{\(a_i\) contains attributable retrieval decision}
            \State Update LinUCB with \((a_i^{\mathrm{ret}},\phi_i,\delta_j,w_{i,j})\)
        \EndIf
        \If{\(a_i\) contains decoding decision}
            \State Apply weighted decoding-bandit update
        \EndIf
        \State Persist trace update \((i,j,\delta_j,E_i,w_{i,j})\)
    \EndFor
\EndFor
\end{algorithmic}
\end{algorithm}

\FloatBarrier

\section{Results}

\subsection{Experimental platform and protocol}

All benchmark runs reported in this manuscript were executed on the author's local workstation, not on a managed cloud service. This matters for reproducibility because wall-clock time is hardware- and runtime-dependent. The same frozen model, target-file convention, validator, and attempt budget were used across the compared conditions. Repeat indices identify repeated local runs, not fixed random seeds; exact stochastic replay is not guaranteed because local model serving and optional bandit sampling were not globally seed-pinned in the recorded artifact. Table~\ref{tab:platform} records the execution setting used for the reported measurements.

\begin{table}[!htbp]
\centering
\caption{Experimental platform and implementation setting used for all reported benchmark runs.}
\label{tab:platform}
\small
\begin{tabularx}{\textwidth}{L{0.31\textwidth}Y}
\toprule
Item & Setting \\
\midrule
Execution host & Single local workstation through WSL2; no cloud execution was used for the reported benchmark runs. \\
Operating system & Ubuntu 24.04.4 LTS on Linux 6.6.87.2-microsoft-standard-WSL2, x86\_64. \\
CPU & 12th Gen Intel Core i7-12700H; 10 physical cores / 20 logical CPUs; 24 MiB L3 cache. \\
Memory & 31 GiB system RAM available inside the WSL2 environment; 8 GiB swap configured. \\
GPU & NVIDIA GeForce RTX 3060 Laptop GPU with 6 GiB VRAM; NVIDIA driver 591.74 reported by \texttt{nvidia-smi}. \\
Runtime stack & Python 3.12.3; Ollama client 0.19.0; local \texttt{qwen3:4b} model served by Ollama during the recorded benchmark runs. \\
Execution mode & Local Python CLI artifact with bounded repair attempts; generated candidates target a single \texttt{algorithm.py} file. \\
Attempt budget & Maximum of three attempts per run in the reported benchmark phases. \\
Validation sequence & Syntax, undefined-name, specification/interface contract, import, runtime, and behavior checks when available. \\
Compared conditions & \REFINE{} self-refinement baseline, \MERA{} proposed stack, and \GRACE{} investigated extension. \\
Primary metrics & Strict validation success, mean attempts, mean wall-clock duration, total score, and primary residual failure type. \\
\bottomrule
\end{tabularx}
\end{table}

Table~\ref{tab:protocol} reports the available benchmark phases. The main comparison is phase1c because it repeats each hard task three times under all three conditions. Phase1b is broader but has one run per task, and phase1a is only a pilot sanity check.

\begin{table}[!htbp]
\centering
\caption{Available experimental phases from the \texttt{manuscript-test-results} artifact directory.}
\label{tab:protocol}
\scriptsize
\begin{tabularx}{\textwidth}{L{0.23\textwidth}ccL{0.23\textwidth}Y}
\toprule
Phase & Tasks & Repeats & Conditions & Role in the manuscript \\
\midrule
phase1a & 2 & 1 & \REFINE{}, \MERA{}, \GRACE{} & Pilot feasibility evidence only; too small for the main claim. \\
phase1b\_rl\_r1 & 11 & 1 & \REFINE{}, \MERA{}, \GRACE{} & Descriptive RL-task comparison; no repeated stochastic trials. \\
phase1c hard repeated & 3 & 3 & \REFINE{}, \MERA{}, \GRACE{} & Strongest available controlled comparison and primary empirical evidence. \\
phase1b\_full\_r1 & not comparable & 1 & \REFINE{} only & Not used as a comparative result because only one condition is present. \\
\bottomrule
\end{tabularx}
\end{table}

The artifact also contains tests, configuration files, validation components, and benchmark scripts. A direct artifact check compiled the source and tests successfully using \texttt{compileall}. Running the full test suite with plugin autoloading disabled collected 203 tests; 201 passed and 2 failed. The two failures are configuration/expectation mismatches: one test still expects \GRACE{} to be enabled by default, while the delivered configuration disables it, and one memory-path test expects deferred rejection although the current implementation raises an immediate workspace-escape error. This supports an implementation claim, but not a fully green artifact claim. The reported benchmark used the direct loop: extracted candidates were written to the target file before validation, and validation determined acceptance rather than write authorization. Staged or pre-materialization design notes in auxiliary artifact documentation should therefore be interpreted as design background, not as the execution path used for the reported measurements.

\begin{table}[!htbp]
\centering
\caption{Artifact-level evidence observed during manuscript preparation.}
\label{tab:artifact}
\small
\begin{tabularx}{\textwidth}{L{0.26\textwidth}L{0.34\textwidth}Y}
\toprule
Check & Observation & Claim implication \\
\midrule
Python compilation & Source and test files compile under \texttt{compileall}. & The Python artifact is syntactically loadable. \\
Pytest suite & 203 collected; 201 passed; 2 failed. & Component evidence is partial; the suite is not fully green. \\
Validation pipeline & The implemented order is syntax, undefined-name, spec-contract, import, runtime, behavior. & Validator-gated repair is implemented. \\
Direct loop & Candidates are written to \texttt{algorithm.py} before validation. & The evaluated loop should not be described as pre-materialization staging. \\
\bottomrule
\end{tabularx}
\end{table}

\subsection{Main comparison: phase1c hard RL}

Table~\ref{tab:phase1c} reports the strongest available comparison. In this hard-task subset, \MERA{} solved 8 of 9 runs, while \REFINE{} and \GRACE{} solved none. \MERA{} also required fewer attempts and less run-level wall-clock time on average. The Wilson score intervals are included only to make the small sample size visible; they are descriptive intervals rather than confirmatory inference or a substitute for a larger repeated-task statistical analysis \cite{f1_3,e4_1,e4_2}.

One \GRACE{} run in phase1c terminated with a client-level error and is retained as a failed run rather than dropped. Thus, the reported 0/9 strict-success result is an end-to-end artifact outcome under the recorded execution protocol, not a claim that every \GRACE{} failure was caused by the same algorithmic mechanism.

The result is also not a clean sweep for the proposed method. \MERA failed one Q-learning repeat in the hard subset, and the broader phase1b comparison is only a descriptive 11-task single-repeat check in which all three conditions were already strong. The honest contribution is therefore modest: the recorded evidence suggests that validation-grounded memory and retrieval control can improve robustness on this local hard subset, not that the controller is universally reliable or that every component is independently necessary.

\begin{table}[!htbp]
\centering
\caption{phase1c hard RL comparison with three tasks, three repeats, and a three-attempt budget.}
\label{tab:phase1c}
\scriptsize
\begin{tabular}{lrrrrr}
\toprule
Condition & Runs & Successes & Success rate & 95\% Wilson CI & Attempts / time \\
\midrule
\REFINE{} & 9 & 0 & 0.000 & [0.000, 0.299] & 3.000 / 700.086 s \\
\MERA{} & 9 & 8 & 0.889 & [0.565, 0.980] & 1.556 / 294.086 s \\
\GRACE{} & 9 & 0 & 0.000 & [0.000, 0.299] & 2.778 / 1367.038 s \\
\bottomrule
\end{tabular}
\end{table}

\begin{figure}[!htbp]
\centering
\includegraphics[width=.98\textwidth]{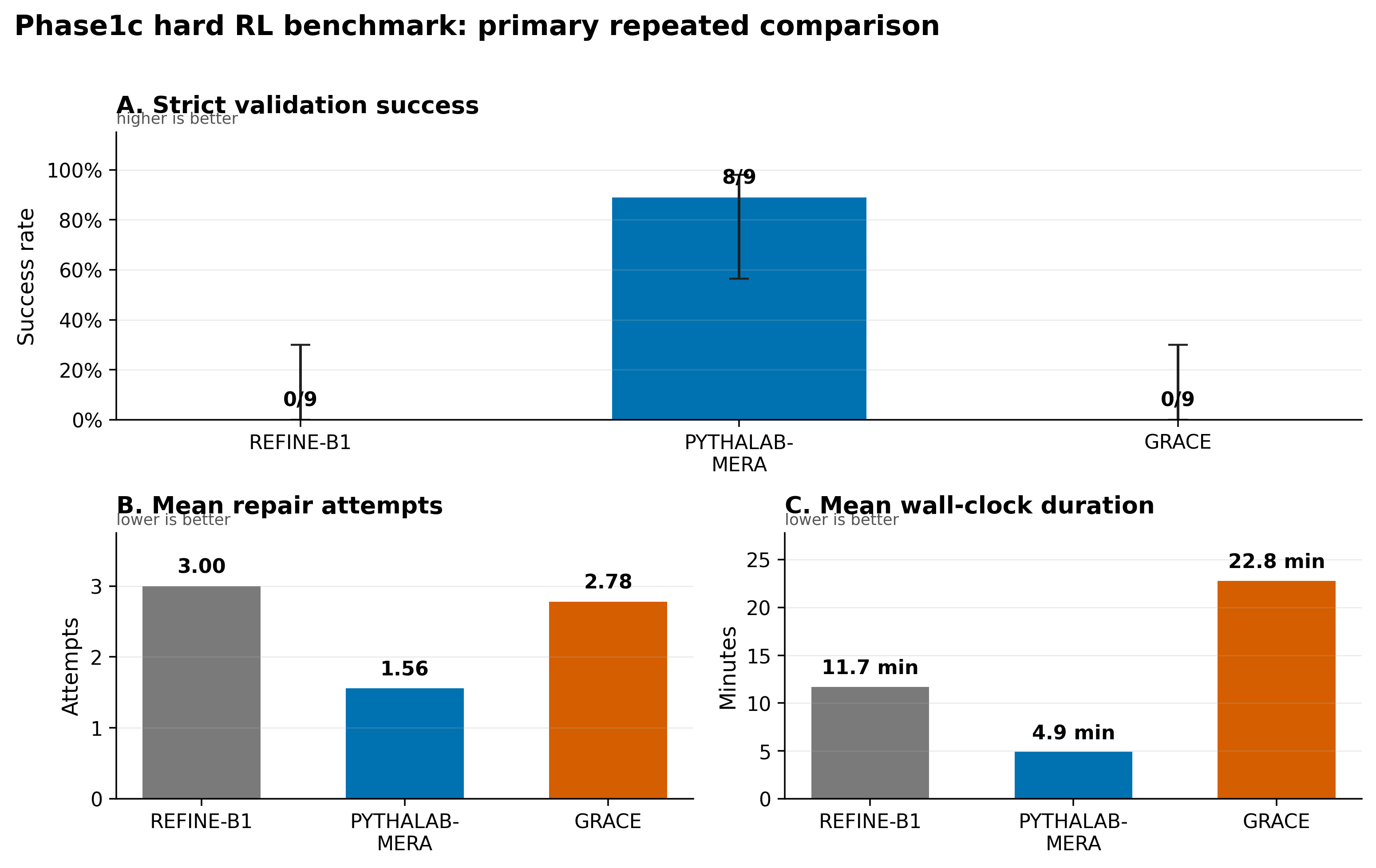}
\caption{Primary phase1c hard RL outcome. Success is reported as strict validator pass rate with 95\% Wilson intervals; attempts and wall-clock time are run-level efficiency measures. The figure shows the same bounded conclusion as the table: \MERA{} is the only condition with nonzero strict success in this measured setting, and it also uses fewer attempts and less time on average.}
\label{fig:phase1c_main}
\end{figure}

The measured outcome supports a narrow claim: in the supplied hard RL subset, the validation-grounded retrieval-memory controller outperformed the compared variants. It does not prove that every internal component is necessary, because the logs do not include separate no-memory, no-LinUCB, no-TD, or no-skill ablations.

\subsection{Per-task behavior and residual failures}

Table~\ref{tab:per_task} gives the phase1c per-task breakdown. \MERA{} solved all SARSA and value-iteration runs and solved two of three Q-learning runs. The remaining \MERA{} failure was runtime-related. \GRACE{} introduced substantial runtime and import-failure overhead, which is why it is treated as an investigated extension rather than the proposed method.

\begin{table}[!htbp]
\centering
\caption{Per-task phase1c hard RL results. Success counts are out of three repeats. Failure codes: RUN = runtime, SEM = semantic, TYPE = type, IMP = import, UNK = unknown.}
\label{tab:per_task}
\scriptsize
\begin{tabularx}{\textwidth}{@{}L{0.16\textwidth}L{0.18\textwidth}rrrrY@{}}
\toprule
Condition & Task & Runs & Succ. & Attempts & Time (s) & Failure types \\
\midrule
\REFINE{} & Q-learning & 3 & 0 & 3.000 & 599.014 & SEM, TYPE \\
\REFINE{} & SARSA & 3 & 0 & 3.000 & 687.325 & RUN, SEM \\
\REFINE{} & Value iteration & 3 & 0 & 3.000 & 813.844 & RUN \\
\MERA{} & Q-learning & 3 & 2 & 1.667 & 317.268 & RUN, UNK \\
\MERA{} & SARSA & 3 & 3 & 2.000 & 303.262 & UNK \\
\MERA{} & Value iteration & 3 & 3 & 1.000 & 261.665 & UNK \\
\GRACE{} & Q-learning & 3 & 0 & 2.333 & 1052.016 & RUN, UNK \\
\GRACE{} & SARSA & 3 & 0 & 3.000 & 2214.300 & IMP, RUN, SEM \\
\GRACE{} & Value iteration & 3 & 0 & 3.000 & 834.717 & RUN \\
\bottomrule
\end{tabularx}
\end{table}

\begin{figure}[!htbp]
\centering
\includegraphics[width=.98\textwidth]{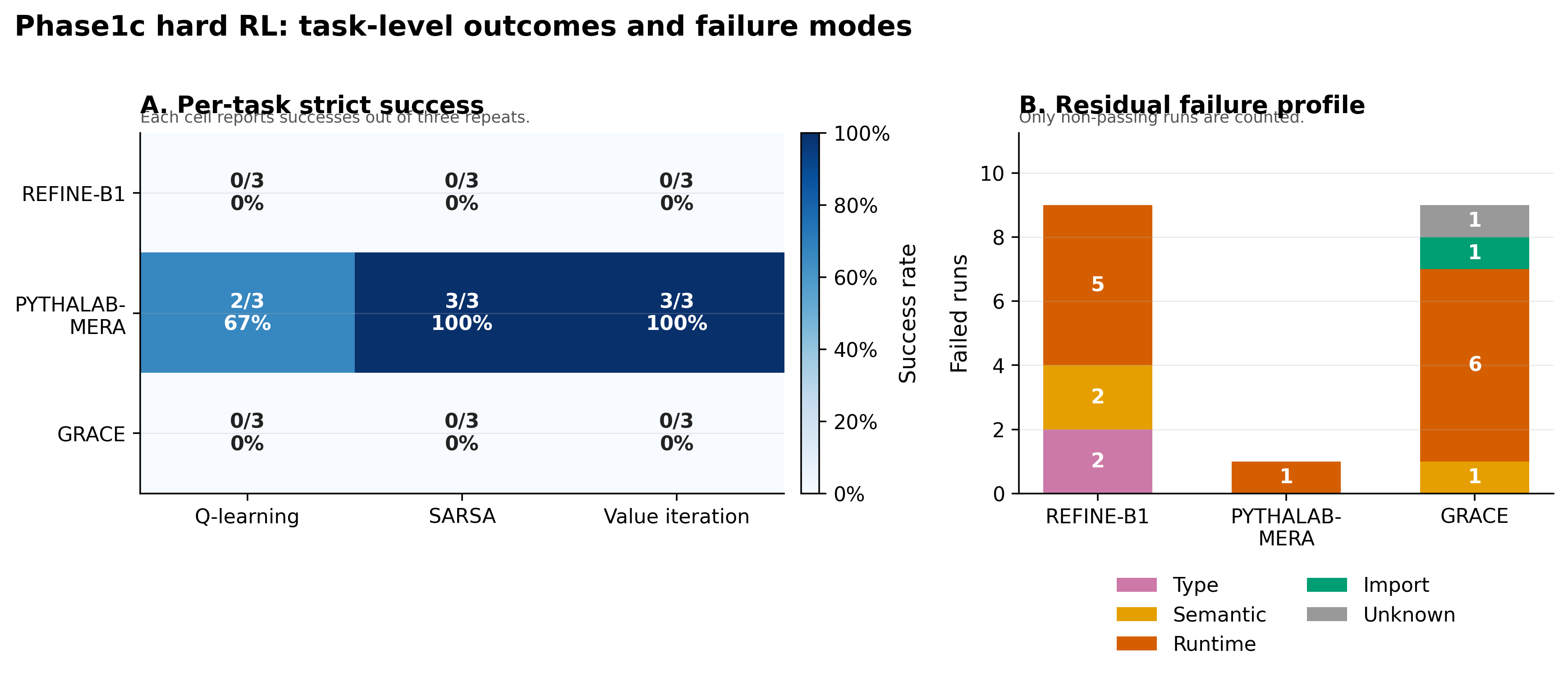}
\caption{Per-task success and residual failure distribution in phase1c. Cell labels show successes out of three repeats; failure bars count only non-passing runs. The profile indicates that \MERA{} reduced hard-task failures in this subset, while \GRACE{} incurred more runtime and import failures.}
\label{fig:per_task_failure}
\end{figure}

Residual failure analysis matters because coding-agent failures are not homogeneous. They can arise from extraction problems, interface mismatches, dependency/import failures, runtime exceptions, or semantic test failures \cite{e6_1,e6_2,e6_3}. In this benchmark, the dominant residual failures for \REFINE{} and \GRACE{} are runtime and semantic failures, whereas \MERA{} has only one failed run in the primary phase.

\subsection{Secondary descriptive comparison: phase1b RL}

Table~\ref{tab:phase1b} reports the 11-task phase1b RL comparison with one run per task. \MERA{} solved all 11 tasks, while \REFINE{} and \GRACE{} each solved 10. This is a small absolute difference, not a decisive benchmark win. Because this phase has no repeated stochastic trials, it is descriptive evidence rather than the main controlled result.

\begin{table}[!htbp]
\centering
\caption{phase1b RL comparison with 11 tasks and one run per task.}
\label{tab:phase1b}
\scriptsize
\begin{tabular}{lrrrrr}
\toprule
Condition & Runs & Successes & Success rate & 95\% Wilson CI & Attempts / time \\
\midrule
\REFINE{} & 11 & 10 & 0.909 & [0.623, 0.984] & 1.455 / 570.295 s \\
\MERA{} & 11 & 11 & 1.000 & [0.741, 1.000] & 1.091 / 262.774 s \\
\GRACE{} & 11 & 10 & 0.909 & [0.623, 0.984] & 1.273 / 262.136 s \\
\bottomrule
\end{tabular}
\end{table}

\begin{figure}[!htbp]
\centering
\includegraphics[width=.98\textwidth]{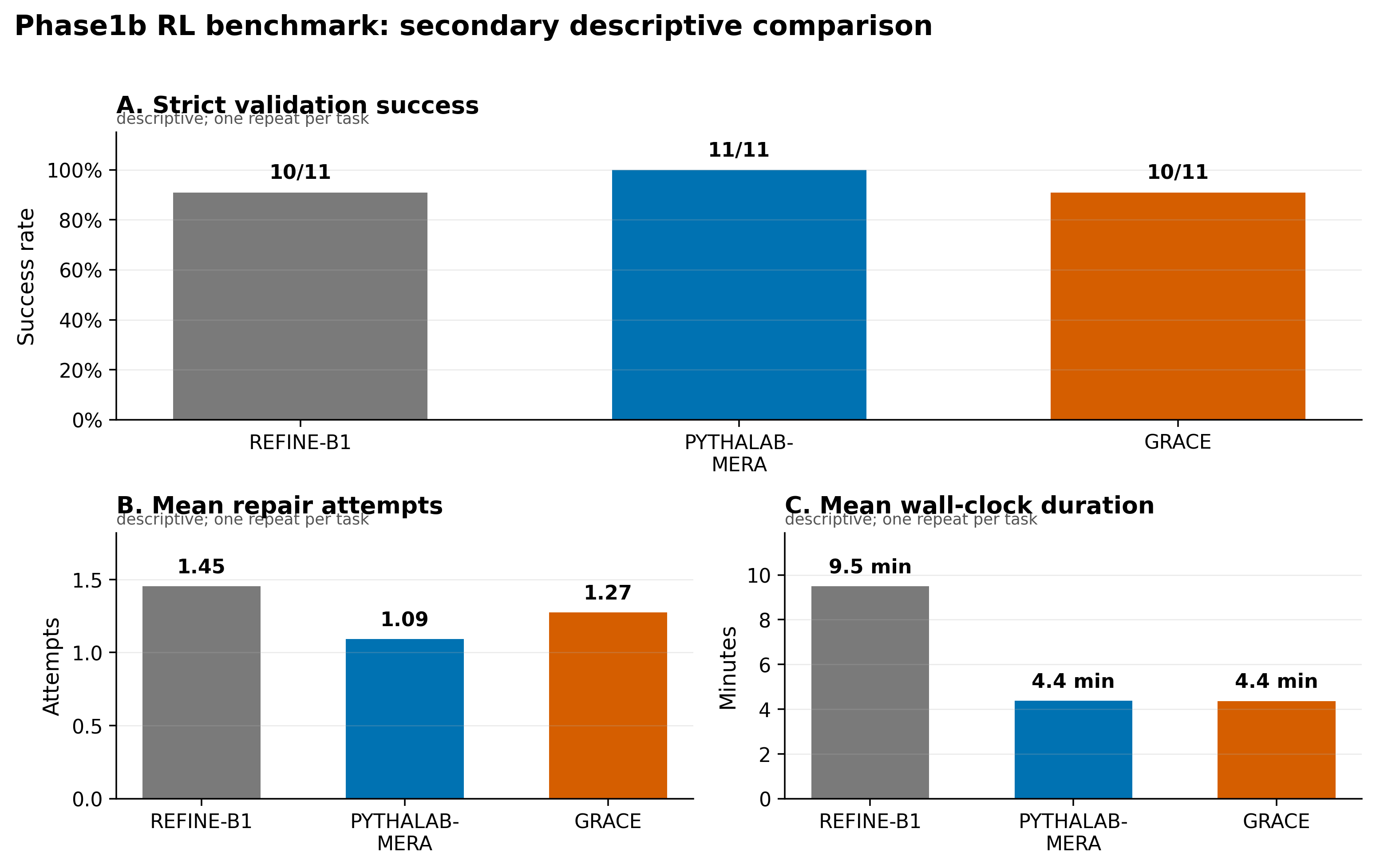}
\caption{Secondary phase1b RL comparison. This plot is descriptive because each of the 11 tasks was run once per condition; it is included to show that the primary phase1c finding is not the only observed efficiency signal, but it is not treated as the main repeated-trial claim.}
\label{fig:phase1b_summary}
\end{figure}

\subsection{Claim gate}

Table~\ref{tab:claim_gate} states what the current evidence supports. The strongest empirical claim is limited to the measured hard RL setting. The current evidence does not isolate each \MERA{} component because no full ablation separating episodic memory, LinUCB retrieval selection, TD credit, skill reuse, and optional decoding control was completed. Prior work motivates such ablations \cite{e2_1,e2_2,e2_3,e2_4}, but literature does not replace project-level ablation evidence.

\begin{table}[!htbp]
\centering
\caption{Claim gate for the current manuscript.}
\label{tab:claim_gate}
\small
\begin{tabularx}{\textwidth}{L{0.33\textwidth}Y}
\toprule
Claim & Current status and allowed wording \\
\midrule
Local validation-conditioned workflow is implemented. & Supported by source, configuration, compile check, and partial tests. Allowed wording: the artifact implements a local validation-conditioned coding-agent workflow. \\
\MERA{} improves over \REFINE{} and \GRACE{} on phase1c hard RL. & Supported only under the measured setting. Allowed wording: \MERA{} outperformed the compared variants in phase1c hard RL. \\
Each \MERA{} component independently improves performance. & Not isolated. Allowed wording: component contributions require future ablations. \\
Broad superiority over self-refinement on all RL tasks. & Unsupported. Allowed wording: the broad phase1b comparison is a descriptive consistency check with a small absolute success-rate difference. \\
TD($\lambda$) credit links later validation outcomes to earlier context decisions. & Supported as implementation and bounded update; no convergence proof. \\
Validator-defined zero cost equals acceptance. & Supported as a formal definition/equivalence under the implemented gate. \\
Safety boundaries reduce execution authority. & Supported only as operational gating; formal safety is not established. \\
Broad general-purpose code-generation superiority. & Unsupported; broader benchmarks and deployments are future work. \\
\bottomrule
\end{tabularx}
\end{table}

\FloatBarrier
\section{Discussion}

The results support a bounded interpretation of \MERA. The clearest finding is that, in the recorded hard RL subset, a frozen-LLM agent equipped with validation-grounded episodic retrieval and delayed credit achieved more strict successes with fewer attempts than the self-refinement baseline and the investigated \GRACE{} extension. This supports the claim that controller-level adaptation over prompt evidence can be useful when validation feedback is available. The claim is intentionally small: the proposed controller appears useful in one local repeated hard-task setting, while still failing one hard Q-learning repeat and while lacking component ablations. It does not show that every component is necessary, and it does not prove that the system generalizes beyond the measured RL task distribution.

Compared with static RAG and passive memory, \MERA changes the role of retrieval. Static RAG systems provide relevant context before generation, but do not generally learn which retrieval mode is useful under a particular failure trajectory. Agentic RAG formalizations argue that retrieval should be treated as a state-dependent action \cite{d1_2}, and frozen-LLM code maintenance systems show that external retrieval and reflection can approach fine-tuning-level utility in some settings \cite{d1_3}. \MERA follows this direction but grounds updates in deterministic validation rather than in LLM-as-judge reflection. This distinction is important because LLM judges for code can exhibit position bias, poor reliability, agreeableness bias, and weak calibration with true correctness \cite{m10_1,m10_2,m10_4,m10_6}.

Compared with self-reflection and self-debugging systems, \MERA remains deliberately external. Intrinsic self-reflection methods can internalize debugging trajectories into model weights \cite{d2_1}, while self-directed debugging can improve repair by exposing richer execution traces \cite{d2_3}. \MERA does not attempt to teach the generator new reasoning behavior. Instead, it uses a frozen model and learns which external evidence to supply. This makes the method lightweight and compatible with local deployment, but it also means the approach cannot fix weaknesses of the underlying model except through better context and feedback.

Compared with automated program repair, \MERA is closer to LLM-based iterative repair than to classical edit-template search. LLM repair work shows that prompt context, fault localization, and iterative validation feedback matter for patch quality \cite{d4_1,d4_2,d4_3}. \MERA contributes a specific controller mechanism for selecting context from persistent memory and skills, rather than a new edit operator or formal APR search space. Therefore, the paper should not claim formal patch correctness or superiority over mature APR systems. The accepted output is validator-passing under the available tests and checks, not necessarily semantically complete for all hidden requirements.

The \GRACE result is scientifically useful despite being negative. \GRACE adds repair-operator consolidation and related mechanisms on top of the proposed controller, but in the primary hard RL setting it achieved no strict successes and incurred much higher mean duration. This supports a conservative architecture lesson: additional structure in an agent loop is not automatically beneficial. This negative result should not be read as a general refutation of structural repair operators, intent--execution-gap features, or consolidation mechanisms. It shows only that, in this implementation, model, task subset, and three-attempt budget, the added \GRACE machinery did not translate into validation success and increased observed cost. Prior work on requirement-driven agents, federated repair, and privacy-preserving repair suggests many plausible extensions \cite{d7_2,d7_4}, but each extension needs controlled evaluation under the same task and cost budget before becoming part of the proposed method.

The safety boundary is also limited. \MERA reduces model execution authority by separating generation from validation, using allowlisted commands, timeouts, and workspace-bound memory paths. However, memory poisoning, prompt-injection through recalled context, and privacy leakage from retrieval datastores remain serious risks \cite{d5_1,d5_3}. The current implementation does not provide differential privacy, cryptographic provenance, immutable audit logs, or formal sandbox proofs. It is therefore best framed as a local, single-tenant research artifact with operational gates, not as a production-safe autonomous coding system.

Finally, the evaluation leaves several important gaps. A full ablation should measure \MERA without episodic memory, without LinUCB retrieval selection, without TD delayed credit, without skill reuse, and with random or heuristic retrieval. A larger benchmark should use more tasks, more repetitions, paired or mixed-effects analysis, and component-level latency instrumentation. The current logs measure run-level duration, but not a detailed breakdown among retrieval, prompt construction, generation, validation, memory persistence, and delayed-credit dispatch. Without these measurements, the result shows end-to-end overhead differences but cannot attribute cost to a specific internal stage.

\section{Conclusion}

This paper presented \MERA, a memory-enhanced validation-conditioned controller for local code generation with a frozen language model. The method formalizes repair as a finite-horizon control problem over prompt evidence, episodic memory, validation-shaped reward, delayed credit, and AST-derived skill reuse. The implementation provides a traceable local artifact with a fail-fast validation pipeline and bounded command execution. In the measured hard RL benchmark, \MERA achieved higher strict success and lower mean attempts than both the self-refinement baseline and the investigated \GRACE{} extension. The supported claim is intentionally narrow and realistic: the implemented controller improved validation success under one recorded local benchmark setting, but it did not solve every run and does not yet establish component-level causal contribution, formal correctness, formal safety, state-of-the-art performance, or general-purpose software synthesis. Future work should prioritize component ablations, larger stochastic evaluations with task-level correlation modeling, latency instrumentation, stronger memory governance, and external benchmarks beyond the current RL task subset.


\begin{thebibliography}{99}
\sloppy
\bibitem{d1_1} Tablan, V., Taylor, S., Hurtado, G., Bernhem, K., Uhrenholt, A., Farei, G., \& Moilanen, K. (2025). Smarter together: Creating agentic communities of practice through shared experiential learning. arXiv. \url{https://arxiv.org/abs/2511.08301}
\bibitem{d1_2} Mishra, S., Niroula, S., Yadav, U., Thakur, D., Gyawali, S., \& Gaire, S. (2026). Sok: Agentic retrieval-augmented generation (rag): Taxonomy, architectures, evaluation, and research directions. arXiv. \url{https://arxiv.org/abs/2603.07379}
\bibitem{d1_3} Shanto, M. H., Asaduzzaman, M., \& Ngom, A. (2026). RAG-Reflect: Agentic Retrieval-Augmented Generation with Reflections for Comment-Driven Code Maintenance on Stack Overflow. arXiv. \url{https://arxiv.org/abs/2604.22217}
\bibitem{d2_1} Jiang, J., Shen, J., Kim, S., Yoo, K. M., Kim, J., \& Kim, S. (2026). ReflexiCoder: Teaching Large Language Models to Self-Reflect on Generated Code and Self-Correct It via Reinforcement Learning. arXiv. \url{https://arxiv.org/abs/2603.05863}
\bibitem{d2_3} Wu, L., Pei, Y., Yang, Z., Li, K., Lu, Z., Tan, H., ... \& Hao, D. (2026). DebugRepair: Enhancing LLM-Based Automated Program Repair via Self-Directed Debugging. arXiv. \url{https://arxiv.org/abs/2604.19305}
\bibitem{d3_1} Sapkota, R., Roumeliotis, K. I., \& Karkee, M. (2025). Vibe coding vs. agentic coding: Fundamentals and practical implications of agentic ai. arXiv. \url{https://arxiv.org/abs/2505.19443}
\bibitem{d4_1} Xia, C. S., Wei, Y., \& Zhang, L. (2023, May). Automated program repair in the era of large pre-trained language models. In 2023 IEEE/ACM 45th International Conference on Software Engineering (ICSE) (pp. 1482-1494). IEEE. \url{https://doi.org/10.1109/ICSE48619.2023.00129}
\bibitem{d4_2} Farzandway, M., \& Ghassemi, F. (2025). Automated repair of c programs using large language models. arXiv. \url{https://arxiv.org/abs/2509.01947}
\bibitem{d4_3} Fan, Z., Gao, X., Mirchev, M., Roychoudhury, A., \& Tan, S. H. (2023, May). Automated repair of programs from large language models. In 2023 IEEE/ACM 45th International Conference on Software Engineering (ICSE) (pp. 1469-1481). IEEE. \url{https://doi.org/10.1109/ICSE48619.2023.00128}
\bibitem{d5_1} Narajala, V. S., \& Narayan, O. (2025). Securing agentic ai: A comprehensive threat model and mitigation framework for generative ai agents. arXiv. \url{https://arxiv.org/abs/2504.19956}
\bibitem{d5_3} Huang, Y., Gupta, S., Zhong, Z., Li, K., \& Chen, D. (2023, December). Privacy implications of retrieval-based language models. In Proceedings of the 2023 Conference on Empirical Methods in Natural Language Processing (pp. 14887-14902). \url{https://doi.org/10.18653/v1/2023.emnlp-main.921}
\bibitem{d7_2} Kuang, S., Tian, Z., Lin, K., Tao, C., Wang, S., Bai, H., ... \& Chen, J. (2026). REAgent: Requirement-Driven LLM Agents for Software Issue Resolution. arXiv. \url{https://arxiv.org/abs/2604.06861}
\bibitem{d7_4} Wang, C., Zhou, Z., Wang, C., Sun, Y., Yang, S., Yuan, Y., ... \& Han, Z. (2025, December). End-to-End Secure Code Repair with Context-Aware Anonymization and Isolated Agent Execution. In 2025 IEEE International Conference on Blockchain Technology and Information Security (ICBCTIS) (pp. 1-8). IEEE. \url{https://doi.org/10.1109/ICBCTIS66509.2025.11387695}
\bibitem{e1_3} Madaan, A., Tandon, N., Gupta, P., Hallinan, S., Gao, L., Wiegreffe, S., ... \& Clark, P. (2023). Self-refine: Iterative refinement with self-feedback. Advances in neural information processing systems, 36, 46534-46594. \url{https://arxiv.org/abs/2303.17651}
\bibitem{e1_4} Gehring, J., Zheng, K., Copet, J., Mella, V., Carbonneaux, Q., Cohen, T., \& Synnaeve, G. (2024). Rlef: Grounding code llms in execution feedback with reinforcement learning. arXiv. \url{https://arxiv.org/abs/2410.02089}
\bibitem{e2_1} Chen, Y., Sun, Y., Wang, H., Zhang, X., Shen, X., Li, W., \& Zhang, W. (2026). Contextual Counterfactual Credit Assignment for Multi-Agent Reinforcement Learning in LLM Collaboration. arXiv. \url{https://arxiv.org/abs/2603.06859}
\bibitem{e2_2} Yang, X., Li, W., Sheng, J., Shen, C., Hua, Y., \& Wang, X. (2025). Agentic Episodic Control. arXiv. \url{https://arxiv.org/abs/2506.01442}
\bibitem{e2_3} Zhang, H., Long, Q., Bao, J., Feng, T., Zhang, W., Yue, H., \& Wang, W. (2026). MemSkill: Learning and Evolving Memory Skills for Self-Evolving Agents. arXiv. \url{https://arxiv.org/abs/2602.02474}
\bibitem{e2_4} Chen, S., Gai, J., Zhou, R., Zhang, J., Zhu, T., Li, J., ... \& Teh, Y. W. (2026). SkillCraft: Can LLM Agents Learn to Use Tools Skillfully?. arXiv. \url{https://arxiv.org/abs/2603.00718}
\bibitem{e4_1} Gonzalez, M. A. A., Hernandez, M. B., Perez, M. A. P., Orozco, B. L., Soto, J. T. C., \& Malagon, S. (2025). Do Repetitions Matter? Strengthening Reliability in LLM Evaluations. arXiv. \url{https://arxiv.org/abs/2509.24086}
\bibitem{e4_2} Gallo, R. J., Baiocchi, M., Savage, T. R., \& Chen, J. H. (2025). Establishing best practices in large language model research: an application to repeat prompting. Journal of the American Medical Informatics Association, 32(2), 386-390. \url{https://doi.org/10.1093/jamia/ocae294}
\bibitem{e6_1} Ning, K., Chen, J., Zhang, J., Li, W., Wang, Z., Feng, Y., ... \& Zheng, Z. (2026). Defining and Detecting the Defects of Large Language Model-Based Autonomous Agents. IEEE Transactions on Software Engineering. \url{https://doi.org/10.1109/TSE.2026.3658554}
\bibitem{e6_2} Chen, Z., Ma, W., \& Jiang, L. (2025). Beyond Final Code: A Process-Oriented Error Analysis of Software Development Agents in Real-World GitHub Scenarios. arXiv. \url{https://arxiv.org/abs/2503.12374}
\bibitem{e6_3} Barke, S., Goyal, A., Khare, A., Singh, A., Nath, S., \& Bansal, C. (2026). AgentRx: Diagnosing AI Agent Failures from Execution Trajectories. arXiv. \url{https://arxiv.org/abs/2602.02475}
\bibitem{e7_2} Zhu, Y., Jin, T., Pruksachatkun, Y., Zhang, A., Liu, S., Cui, S., Kapoor, S., Longpre, S., Meng, K., Weiss, R., Barez, F., Gupta, R., Dhamala, J., Merizian, J., Giulianelli, M., Coppock, H., Ududec, C., Sekhon, J., Steinhardt, J., Kellermann, A., Schwettmann, S., Zaharia, M., Stoica, I., Liang, P., \& Kang, D. (2025). Establishing best practices for building rigorous agentic benchmarks. arXiv. \url{https://arxiv.org/abs/2507.02825}
\bibitem{f1_1} Li, L., Chu, W., Langford, J., \& Schapire, R. E. (2010). A contextual-bandit approach to personalized news article recommendation. In \textit{Proceedings of the 19th International Conference on World Wide Web} (pp. 661--670). ACM. \url{https://doi.org/10.1145/1772690.1772758}
\bibitem{f1_2} Sutton, R. S. (1988). Learning to predict by the methods of temporal differences. \textit{Machine Learning, 3}, 9--44. \url{https://doi.org/10.1007/BF00115009}
\bibitem{f1_3} Wilson, E. B. (1927). Probable inference, the law of succession, and statistical inference. \textit{Journal of the American Statistical Association, 22}(158), 209--212. \url{https://doi.org/10.1080/01621459.1927.10502953}
\bibitem{i1_1} Gautam, D., Garg, S., Jang, J., Sundaresan, N., \& Zilouchian Moghaddam, R. (2025). RefactorBench: Evaluating stateful reasoning in language agents through code. arXiv. \url{https://arxiv.org/abs/2503.07832}
\bibitem{i2_1} Arimbur, J. J. (2026). How Many Tries Does It Take? Iterative Self-Repair in LLM Code Generation Across Model Scales and Benchmarks. arXiv. \url{https://arxiv.org/abs/2604.10508}
\bibitem{i3_1} Dai, D., Liu, M., Li, A., Cao, J., Wang, Y., Wang, C., ... \& Zheng, Z. (2025). Feedbackeval: A benchmark for evaluating large language models in feedback-driven code repair tasks. arXiv. \url{https://arxiv.org/abs/2504.06939}
\bibitem{i4_1} Sunil, B. D., Sinha, I., Maheshwari, P., Todmal, S., Mallik, S., \& Mishra, S. (2026). Memory poisoning attack and defense on memory based llm-agents. arXiv. \url{https://arxiv.org/abs/2601.05504}
\bibitem{i4_2} Srivastava, S. S., \& He, H. (2025). MemoryGraft: Persistent compromise of LLM agents via poisoned experience retrieval. arXiv. \url{https://arxiv.org/abs/2512.16962}
\bibitem{i4_4} Liu, F., Zhang, Y., Luo, J., Dai, J., Chen, T., Yuan, L., Yu, Z., Shi, Y., Li, K., Zhou, C., Chen, H., \& Yang, M. (2025). Make agent defeat agent: Automatic detection of taint-style vulnerabilities in LLM-based agents. In 34th USENIX Security Symposium (USENIX Security 25). \url{https://www.usenix.org/conference/usenixsecurity25/presentation/liu-fengyu}
\bibitem{i5_1} Wang, W., Wang, Y., Joty, S., \& Hoi, S. C. (2023, November). Rap-gen: Retrieval-augmented patch generation with codet5 for automatic program repair. In Proceedings of the 31st ACM Joint European Software Engineering Conference and Symposium on the Foundations of Software Engineering (pp. 146-158). \url{https://doi.org/10.1145/3611643.3616256}
\bibitem{i5_4} Lee, H., \& Yang, G. (2025, November). AgentRepair: Multi-Agent, AST-Anchored, Retrieval-Augmented Program Repair for Cold-Start Environments. In 2025 12th International Conference on Dependable Systems and Their Applications (DSA) (pp. 121-132). IEEE. \url{https://doi.org/10.1109/DSA66321.2025.00025}
\bibitem{i5_6} Chondamrongkul, N., Kyaw, M. P. P., Ko, S. M., Paing, P. P., Swe, M. K. T., \& Hongthong, T. (2026). RepoAI: Automated code refactoring through multi-agent LLM orchestration and retrieval-augmented generation. Science of Computer Programming, 253, Article 103477. \url{https://doi.org/10.1016/j.scico.2026.103477}
\bibitem{i6_2} Bui, N. D. (2026). Building Effective AI Coding Agents for the Terminal: Scaffolding, Harness, Context Engineering, and Lessons Learned. arXiv. \url{https://arxiv.org/abs/2603.05344}
\bibitem{i6_6} Rombaut, B. (2026). Inside the Scaffold: A Source-Code Taxonomy of Coding Agent Architectures. arXiv. \url{https://arxiv.org/abs/2604.03515}
\bibitem{m1_1} Kim, M. H. (2025). Bridging Symbolic Control and Neural Reasoning in LLM Agents: The Structured Cognitive Loop. arXiv. \url{https://arxiv.org/abs/2511.17673}
\bibitem{m2_1} Poon, M., Dai, X., Liu, X., Kong, F., Lui, J. C., \& Zuo, J. (2026, March). Online multi-llm selection via contextual bandits under unstructured context evolution. In Proceedings of the AAAI Conference on Artificial Intelligence (Vol. 40, No. 29, pp. 24855-24863). \url{https://doi.org/10.1609/aaai.v40i29.39672}
\bibitem{m2_4} Hong, Z., Zhang, Q., Sun, J., Shang, Z., Kong, M., Wang, X., ... \& Dai, Z. (2026). MASPOB: Bandit-Based Prompt Optimization for Multi-Agent Systems with Graph Neural Networks. arXiv. \url{https://arxiv.org/abs/2603.02630}
\bibitem{m2_5} Rietz, F., Smirnov, O., Karimi, S., \& Cao, L. (2026). Prompt Tuning Decision Transformers with Structured and Scalable Bandits. Advances in Neural Information Processing Systems, 38, 58258-58286. \url{https://www.microsoft.com/en-us/research/publication/prompt-tuning-decision-transformers-with-structured-and-scalable-bandits/}
\bibitem{m3_1} Sloan, M., \& Wang, J. (2015, September). Dynamic information retrieval: Theoretical framework and application. In Proceedings of the 2015 International Conference on the theory of Information Retrieval (pp. 61-70). \url{https://doi.org/10.1145/2808194.2809457}
\bibitem{m3_3} Yang, A., \& Yang, G. H. (2017, October). A contextual bandit approach to dynamic search. In Proceedings of the ACM SIGIR International Conference on Theory of Information Retrieval (pp. 301-304). \url{https://doi.org/10.1145/3121050.3121101}
\bibitem{m3_4} Zhang, W., Zhu, Y., Lu, Y., Demarne, M., Wang, W., Deng, K., ... \& Krishnan, S. (2025, November). Flair: Feedback learning for adaptive information retrieval. In Proceedings of the 34th ACM International Conference on Information and Knowledge Management (pp. 6284-6292). \url{https://doi.org/10.1145/3746252.3761553}
\bibitem{m4_3} Le, H., Wang, Y., Gotmare, A. D., Savarese, S., \& Hoi, S. C. H. (2022). Coderl: Mastering code generation through pretrained models and deep reinforcement learning. Advances in Neural Information Processing Systems, 35, 21314-21328. \url{https://arxiv.org/abs/2207.01780}
\bibitem{m4_5} Yu, Z., Gu, W., Wang, Y., Jiang, X., Zeng, Z., Wang, J., ... \& Zhang, S. (2024). Reasoning through execution: Unifying process and outcome rewards for code generation. arXiv. \url{https://arxiv.org/abs/2412.15118}
\bibitem{m5_4} Han, B., Ren, Z., Wu, Z., Zhou, Y., \& Peng, J. (2022). Off-policy reinforcement learning with delayed rewards. In Proceedings of the 39th International Conference on Machine Learning (PMLR, Vol. 162, pp. 8280--8303). \url{https://proceedings.mlr.press/v162/han22e.html}
\bibitem{m5_6} Li, B., Sun, Z., Huang, T., Zhang, H., Wan, Y., Li, G., ... \& Lyu, C. (2024). Ircoco: Immediate rewards-guided deep reinforcement learning for code completion. Proceedings of the ACM on Software Engineering, 1(FSE), 182-203. \url{https://doi.org/10.1145/3643735}
\bibitem{m6_1} Zhang, D., Chen, L., Zhang, S., Xu, H., Zhao, Z., \& Yu, K. (2023). Large language models are semi-parametric reinforcement learning agents. Advances in Neural Information Processing Systems, 36, 78227-78239. \url{https://arxiv.org/abs/2306.07929}
\bibitem{m6_5} Krishnamoorthy, A., Ivatury, K., \& Ahmadnia, B. (2025, September). Multi-Agent Reinforcement Learning for Interactive Code Debugging with Human Feedback and Memory. In Proceedings of the 15th International Conference on Recent Advances in Natural Language Processing-Natural Language Processing in the Generative AI Era (pp. 595-603). \url{https://doi.org/10.26615/978-954-452-098-4-070}
\bibitem{m7_2} Chilowicz, M., Duris, E., \& Roussel, G. (2009, May). Syntax tree fingerprinting for source code similarity detection. In 2009 IEEE 17th international conference on program comprehension (pp. 243-247). IEEE. \url{https://doi.org/10.1109/ICPC.2009.5090050}
\bibitem{m7_3} Li, Y., Wang, S., \& Nguyen, T. (2021, May). Fault localization with code coverage representation learning. In 2021 IEEE/ACM 43rd International Conference on Software Engineering (ICSE) (pp. 661-673). IEEE. \url{https://doi.org/10.1109/ICSE43902.2021.00067}
\bibitem{m7_4} Verma, A., Udhayanan, P., Shankar, R. M., Kn, N., \& Chakrabarti, S. K. (2021, October). Source-code similarity measurement: syntax tree fingerprinting for automated evaluation. In Proceedings of the First International Conference on AI-ML Systems (pp. 1-7). \url{https://doi.org/10.1145/3486001.3486228}
\bibitem{m8_1} Ellis, K., Morales, L., Sabl\'{e}-Meyer, M., Solar-Lezama, A., \& Tenenbaum, J. (2018). Learning libraries of subroutines for neurally guided Bayesian program induction. In Advances in Neural Information Processing Systems 31 (pp. 7816--7826). \url{https://papers.nips.cc/paper/8006-learning-libraries-of-subroutines-for-neurallyguided-bayesian-program-induction}
\bibitem{m8_4} Xu, P., Wu, G., Chen, X., Yu, T., Xiao, C., Dernoncourt, F., ... \& Swaminathan, V. (2026, March). Skill discovery for software scripting automation via offline simulations with llms. In Findings of the Association for Computational Linguistics: EACL 2026 (pp. 743-759). \url{https://doi.org/10.18653/v1/2026.findings-eacl.37}
\bibitem{m8_6} Stengel-Eskin, E., Prasad, A., \& Bansal, M. (2024). Regal: Refactoring programs to discover generalizable abstractions. arXiv. \url{https://arxiv.org/abs/2401.16467}
\bibitem{m9_3} Rabin, R., Hostetler, J., McGregor, S., Weir, B., \& Judd, N. (2025). Sandboxeval: Towards securing test environment for untrusted code. arXiv. \url{https://arxiv.org/abs/2504.00018}
\bibitem{m9_4} Wang, J., Luo, X., Cao, L., He, H., Huang, H., Xie, J., ... \& Cai, Y. (2024). Is your ai-generated code really safe? evaluating large language models on secure code generation with codeseceval. arXiv. \url{https://arxiv.org/abs/2407.02395}
\bibitem{m10_1} Jiang, H., Chen, Y., Cao, Y., Lee, H. Y., \& Tan, R. T. (2025). Codejudgebench: Benchmarking llm-as-a-judge for coding tasks. arXiv. \url{https://arxiv.org/abs/2507.10535}
\bibitem{m10_2} Zhao, Y., Luo, Z., Tian, Y., Lin, H., Yan, W., Li, A., \& Ma, J. (2024). CodeJudge-Eval: Can large language models be good judges in code understanding? arXiv. \url{https://arxiv.org/abs/2408.10718}
\bibitem{m10_4} Jain, S., Ahmed, U. Z., Sahai, S., \& Leong, B. (2025). Beyond consensus: Mitigating the agreeableness bias in llm judge evaluations. arXiv. \url{https://arxiv.org/abs/2510.11822}
\bibitem{m10_6} Spiess, C., Gros, D., Pai, K. S., Pradel, M., Rabin, M. R. I., Alipour, A., ... \& Ahmed, T. (2025, April). Calibration and correctness of language models for code. In 2025 IEEE/ACM 47th International Conference on Software Engineering (ICSE) (pp. 540-552). IEEE. \url{https://doi.org/10.1109/ICSE55347.2025.00040}
\bibitem{m11_1} Chen, Z., Chen, D., Jin, R., Liang, Y., Xie, Y., \& Sun, H. (2026). Bridging Online and Offline RL: Contextual Bandit Learning for Multi-Turn Code Generation. arXiv. \url{https://arxiv.org/abs/2602.03806}
\end{thebibliography}
\end{document}